\documentclass[12pt,a4paper]{article}
\usepackage[title]{appendix}
\usepackage{amsmath}
\usepackage{amsthm}
\usepackage{amssymb}
\usepackage{bbm}
\usepackage{verbatim}
\usepackage{graphicx}
\usepackage{subfigure}
\usepackage{afterpage}
\usepackage{etoolbox}
\usepackage{float}
\usepackage{rotating}
\usepackage{algorithmic}
\usepackage[algosection,algoruled,vlined]{algorithm2e}
\usepackage{cite}
\usepackage{fullpage}
\usepackage{float}
\usepackage{url}

\numberwithin{equation}{section}
\numberwithin{figure}{section}
\numberwithin{table}{section}

\makeatletter
\renewcommand{\p@subfigure}{\thefigure}
\makeatother

\newtheorem{theorem}{Theorem}[section]
\newtheorem{proposition}[theorem]{Proposition}

\newtheorem{observation}[theorem]{Observation}

\newcommand{\repeatable}[2]{\makeatletter \global\expandafter\def\csname repText@#1\endcsname {#2} \makeatother #2}
\newcommand{\repeatxt}[1]{\makeatletter \expandafter\csname repText@#1\endcsname \makeatother}

\newtoggle{citeInAbstract}
\toggletrue{citeInAbstract}

\newcommand{\usecrop}[2]
{
	\newlength{\cropwidth}
	\setlength{\cropwidth}{\the\textwidth}
	\addtolength{\cropwidth}{#1}
	\newlength{\cropheight}
	\setlength{\cropheight}{\the\textheight}
	\addtolength{\cropheight}{#2}
	\usepackage[width=\the\cropwidth,height=\the\cropheight,center]{crop}
}

\DeclareMathAlphabet{\mathpzc}{OT1}{pzc}{m}{it}

\newcommand{\abs}[1]{\left | #1 \right |}

\newcommand{\norm}[1]{\left \| #1 \right \|}
\newcommand{\norminline}[1]{\| #1 \|}
\newcommand{\Rn}[1]{\mathbb{R}^{#1}}

\usepackage{url}

\newtheorem{thm}{Theorem}[section]
\newtheorem{lem}[thm]{Lemma}
\newtheorem{defa}[thm]{Definition}

\usecrop{1.69in}{1.5in}

\title{PCA-Based Out-of-Sample Extension for Dimensionality Reduction}
\author{Yariv Aizenbud \and Amit Bermanis \and Amir Averbuch}

\begin{document}

\maketitle

\begin{abstract}
Dimensionality reduction methods are very common in the field of
high dimensional data analysis. Typically, algorithms for
dimensionality reduction are computationally expensive. Therefore,
their applications for the analysis of  massive amounts of data are
impractical. For example, repeated computations due to accumulated
data are computationally prohibitive. In this paper, an
out-of-sample extension scheme, which is used as a complementary
method for dimensionality reduction, is presented.  We describe an
algorithm which performs an out-of-sample extension to newly-arrived
data points. Unlike other extension algorithms such as Nystr\"om
algorithm, the proposed algorithm uses the intrinsic geometry of the
data and properties for dimensionality reduction map. We prove that
the error of the proposed algorithm is bounded. Additionally to the
out-of-sample extension, the algorithm provides a degree of the
abnormality of any newly-arrived data point.
\end{abstract}

\section{Introduction}

Analysis of large amounts of high-dimensional big data is of great
interest since it illuminates the underlying phenomena. To cope with
high-dimensional big data, it is sometimes assumed that there are
some (unobservable) dependencies between the parameters of the
multidimensional data points. Mathematically, it means that the data
is sampled from a low-dimensional manifold that is embedded in a
high dimensional ambient space. Dimensionality reduction methods,
which rely on the presence of a manifold, map the data into a
low-dimensional space while preserving certain qualities of the
low-dimensional structures of the data.

A broad class of dimensionality reduction methods are based on
kernel-based methods. The kernel encapsulates a measure of mutual
affinities (or similarities) between data points. Particularly, if
the kernel is semi-positive definite, it can be considered as Gram
matrix of inner products, which correspond to an implicit mapping of
the data to a high dimensional space, typically refereed to as the
feature space. Depending on the chosen kernel, the new geometry of
the data in feature space, represents important features of the
data.

Kernel-PCA is a technique that generalizes the well known principal
component analysis (PCA)~\cite{jolliffe:PCA, hotelling:PCA}. While
the latter detects principal directions of data in Euclidean space
and then  projects the data onto them, the former does the same in
the feature space. It is resulted in a low dimensional Euclidean
representation (embedding) of the data that approximates the feature
space geometry. The dimensionality of the embedding space is
affected by the decay rate of the kernel's spectrum. Examples of
kernel methods are diffusion maps (DM)~\cite{coifman:DM}, local
linear embedding (LLE)~\cite{roweis:LLE}, Laplacian
eigenmaps~\cite{belkin:LaplacianEigenmaps}, Hessian
eigenmaps~\cite{donoho:HessianEigenmaps} and local tangent space
alignment~\cite{yang:LTSA,zhang:LTSA}.

From a practical point of view, kernel methods have a significant
computational drawback: spectral analysis of the kernel matrix
becomes impractical for large datasets due to high computational
complexity required to manipulate a kernel matrix. Their global
nature is also disadvantageous. Furthermore, in many applications,
the analysis process is dynamic due to data accumulation over time
and, as a result, the embedding has to be modified once in a while.
Processing a kernel matrix  in memory becomes
impractical for large datasets due to their sizes.

A general solution scheme embeds a subset of the source data that is
usually referred to as a training dataset. Then, the embedding is
extended to any out-of-sample data point. The Nystr\"om
method~\cite{delves, baker, num_rec}, which is widely used in
integral equations solvers, has become very popular as an
out-of-sample extension method associated with  dimensionality
reduction methodology. For a review of spectral clustering and
Nystr\"om extension see Section $2$ in~\cite{multi_sample}. The Nystr\"om extension scheme has three significant disadvantages: (a) It requires diagonalization of a matrix that costs $O(n^{3})$  operations \cite{golub}.
(b) It requires working with a matrix which may be ill-conditioned due to fast decay of its spectrum, and (c) it is unclear how to choose the length parameter $\epsilon$ since the output is sensitive to the choice of $\epsilon$.
Some limitations of Nystr\"om extension are overcome in \cite{ bermanis:multiSampExtACHA}.
Geometric Harmonics (GH)~\cite{coifman:GH} is another out-of-sample
extension method. It uses the Nystr\"om extension of eigenfunctions
of a kernel defined on the data. In order to avoid numerical
instabilities, it uses only the significant spectral components. In
that sense, the GH framework filters out high frequencies, which are
determined by the kernel, rather than by the data. This problem,
additionally to the fixed interpolation distance problem, is treated
in~\cite{laplacian_pyramids}, where a multiscale interpolation
scheme is introduced. Another multiscale approach, which aims to
solve the aforementioned limitations, was recently introduced
in~\cite{bermanis:multiSampExtACHA}. Both methods project the
objective function on the eigencomponents of a series of kernels,
which cover the complete spectrum of that function. The difference
between these methods is in the extraction of principal components
while the former is spectral and  the latter is interpolative.

All these methods use a kernel matrix (or, perhaps, its low rank approximation) as an interpolation matrix. This mechanism is strongly related to a variety of isotropic interpolation methods that employ radial basis functions (RBF). Such methods are used for scattered data approximation, where the data lies in a metric space. More details about RBF and scattered data approximation can be found in~\cite{RBF} and~\cite{scatter},
respectively.

In this paper, we employ the manifold assumption to establish an
anisotropic out-of-sample extension. We suggest a new anisotropic
interpolation scheme that ascribes for each data point a likelihood
neighborhood. This likelihood is based on geometric features from
the dimensionality reduction map by using PCA of the map\rq{}s
image. Incorporation of such neighborhood information produces a
linear system for finding the out-of-sample extension for this data
point. This method also provides an abnormality measure for a
newly-mapped data point.

The paper has the following structure: Section \ref{sec:setup}
introduces the problem and the needed definitions. The construction
of the out-of-sample extension is described in section
\ref{sec:alg_const}.
 Section~\ref{sec:feature based variance} establishes the geometric-based stochastic linear system
 on which the interpolant is based. Three different interpolants are presented where each is based on different geometric considerations.
  In Section \ref{sec:bound}, an analysis of interpolation\rq{}s error is presented for the case of Lipschitz mappings.
  Computational complexity analysis of the scheme is presented in Section~\ref{sec:complexity}. Experimental results for both synthetic data and real-life data are presented in Section~\ref{sec:examples}.


\section{Problem Setup}
\label{sec:setup}

Let $\mathcal{M}$ be a compact low-dimensional manifold of intrinsic
dimension $m$ that lies in a high-dimensional ambient space $\Rn{n}$
($m<n$), whose Euclidean metric is denoted by $\norm{\cdot}$. Let
$\psi$ be a smooth, Lipschitz  and dimensionality reducing function
defined on $\mathcal{M}$, i.e. $\psi:\mathcal{M} \to
\mathcal{N}\subset \Rn{d}$ ($m<d< n$), where $\mathcal{N}$ is a
$m$-dimensional manifold. Let $M=\{x_1,\ldots,x_p\}\subset
\mathcal{M}$ be a finite training dataset, sufficiently dense
sampled from $\mathcal{M}$, whose image
$\psi(M)=\{\psi(x_1),\ldots,\psi(x_p)\}$ under $\psi$ was already
computed. Given an out-of-sample data point
$x\in\mathcal{M}\backslash M$, we aim to embed it into $\Rn{d}$
while preserving some local properties of $\psi$. The embedding of
$x$ into $\Rn{d}$ is denoted by $\hat{\psi}(x)$. It is referred to
as the extension of $\psi$ to $x$.

The proposed extension scheme is based on the local geometric
properties of $\psi$ in the neighborhood of $x$, denoted by $N(x)$.
Specifically, the influence of a neighbor $x_j\in N(x)$ on the value
of $\hat{\psi}(x)$ depends on its distance from $x$ and the geometry
of the image $\psi(N(x))$ of $N(x)$ under $\psi$. This approach is
reflected by considering $\hat{\psi}(x)$ as a random variable with
mean $\mathbb{E}\hat{\psi}(x)=\psi(x_j)$ and a variance
$\mathbb{V}\hat{\psi}(x)$ that depends on both the distance of $x$
from $x_j$ and on some geometric properties of $\psi(N(x))$ that
will be detailedly discussed in Section~\ref{sec:tangent space}.
Mathematically,
\begin{equation}\label{eq:basic_rv}
\hat{\psi}(x) = \psi(x_j)+\omega_j,
\end{equation}
where $\omega_j$ is a random variable with mean
$\mathbb{E}\omega_j=0$ and variance $\mathbb{V}\omega_j=\sigma_j$
that, as previously mentioned,  depends on the local geometry of
$\psi$ in the neighborhood $N(x)$ of $x$. Thus, we get $\abs{N(x)}$
equations for evaluating  $\hat{\psi}(x)$, one for each neighbor
$x_j\in N(x)$. The optimal solution then, is achieved by the
generalized least squares approach described in
Section~\ref{sec:GLS}.

\subsection{Generalized Least Squares (GLS)}\label{sec:GLS}
In this section, we briefly describe the GLS approach that will be
utilized to evaluate $\hat{\psi}(x)$.  In general, the GLS addresses
the problem of a linear regression that assumes neither independence
nor common variance between the random variables. Thus, if
$y=(y_1,\ldots,y_k)^T$ are random variables that correspond to $k$
data points in $\mathbbm{R}^d$, the addressed regression problem is
\begin{equation}\label{eq:regression}
y=X\beta+\mu,
\end{equation}
where $X$ is an $k\times d$ matrix that stores the data points as
its rows and $\mu\in\Rn{k}$ is an error vector. Respectively to the
aforementioned assumption, the $k\times k$ conditional covariance
matrix of the error term $W=\mathbb{V}\{\mu|X\}$ is not necessarily
scalar or diagonal. The GLS solution to Eq.~\ref{eq:regression} is
\begin{equation}\label{eq:gls_solution}
\hat{\beta}= (X^TW^{-1}X)^{-1}X^TW^{-1}y.
\end{equation}
The Mahalanobis distance between two random vectors $v_1$ and $v_2$ of the same distribution with conditional covariance matrix $W$ is
\begin{equation}\label{eq:Mahal_dist}
\norm{v_1-v_2}_W\triangleq \sqrt{(v_1-v_2)^T W^{-1}(v_1-v_2)}.
\end{equation}
\begin{observation}\label{obs:mahlanobis}
The Mahalanobis distance in Eq.~\ref{eq:Mahal_dist} measures the
similarity between $v_1$ and $v_2$ in respect to $W$. If the random
variables are independent, then $W$ is diagonal. Then, it is more
affected by low variance random variables and less by high variance
variables.
\end{observation}
The GLS solution from Eq.~\ref{eq:gls_solution}, minimizes the squared Mahalanobis distance between $y$ and the estimator $X\beta$, i.e.
\begin{equation}\label{eq: mahalanobis}
\hat{\beta}=\arg\min_{\beta\in\Rn{d}} \norm{y-X\beta}_W.
\end{equation}
Further details concerning GLS can for example  be found
in~\cite{gls}.

In our case, for a fixed out-of-sample data point
$x\in\mathcal{M}\backslash M$ with its $k=\abs{N(x)}$ neighbors, a
linear system of $k$ equations, each of the form of
Eq.~\ref{eq:basic_rv}, is solved for $\hat{\psi}(x)$. Without loss
of generality, we assume that $N(x)=\{x_1,\ldots,x_k\}$. The matrix
formulation of such a system is
\begin{equation}\label{eq:system_rv}
J\hat{\psi}(x) = \Psi+\Omega,
\end{equation}
where $J=[I_d,\ldots,I_d]^T$ is the $kd\times d$ identity blocks
matrix, $\Omega$ is a $kd$-long vector, whose $j$-th section is the
$d$-long constant vector $(\omega_j,\ldots,\omega_j)^T$ and $\Psi$
is a $kd$-long vector, whose $j$-th section is the $d$-long vector
$\psi(x_j)$. The vector $\Psi$ encapsulates the images of $N(x)$
under $\psi$ such as the neighborhood of $\hat{\psi}(x)$ in
$\mathcal{N}$. The corresponding covariance matrix is the $kd\times
kd$ blocks diagonal matrix $W$,
\begin{equation}\label{eq:W}
W = diag(w_1,\ldots,w_k),
\end{equation}
whose $j$-th diagonal element is $w_j=\sigma_j^2I_d$. Therefore, due to Eq.~\ref{eq:gls_solution}, the GLS solution to Eq.~\ref{eq:system_rv} is
\begin{equation}\label{eq:psi_hat_solution}
\hat{\psi}(x)\triangleq (J^TW^{-1}J)^{-1}J^TW^{-1}\Psi,
\end{equation}
and it minimizes the  Mahalanobis distance
\begin{equation}\label{eq:maha measure}
m(x)\triangleq\norminline{J\hat{\psi}(x)-\Psi}_W
\end{equation}
that measures the similarity (with respect to $W$) between
$\hat{\psi}(x)$ and its neighbors $\{\psi(x_1),\ldots,\psi(x_k)\}$
in $\mathcal{N}$, which are encapsulated in $\Psi$. Once $W$ is
defined as an invertible covariance matrix $\hat{\psi}(x)$, as
defined in Eq.~\ref{eq:psi_hat_solution}, is well posed. The
definition of $W$ depends on the definition of $w_j$ for any
$j=1,\ldots,k$, which can be chosen subjected to the similarity
properties to be preserved by $\psi$. These properties are discussed
in Section~\ref{sec:feature based variance}. Once
Eq.~\ref{eq:psi_hat_solution} is solved for $\hat{\psi}(x)$, the
Mahalanobis distance from Eq.~\ref{eq:maha measure} provides a
measure for the disagreement between the out-of-sample extension of
$\psi$ and $x$ with the surrounding geometry. Thus, a large value of
$m(x)$ (Eq. \ref{eq:maha measure}) indicates that $x$ resides
outside of $\mathcal{M}$ and thus, in data analysis terminology, it
can be considered as an anomalous data point.

\section {Construction of the out-of-sample extension}
\label{sec:alg_const} As mentioned in Section~\ref{sec:GLS}, the GLS
solution minimizes the Mahalanobis distances between $\hat{\psi}(x)$
and its neighbors according to the stored information in $W$. Thus,
if the variances are determined subjected to some feature, then
$\hat{\psi}(x)$, which is defined in Eq.~\ref{eq:psi_hat_solution},
is the closest point in $\mathcal{N}$ to its neighbors with respect
to this feature.

The idea of Algorithm \ref{alg:pbe} is to get  a linear
approximation for the local geometry of $\psi$ and then device an
out-of-sample extension that best preserves that linear demand using
GLS. The GLS solution also provides the error, which, as described
in section \ref{sec:GLS}, is considered as an anomalous score.

\begin{algorithm}[H]
    \caption{PCA-Based Out-Of-Sample Extension}
    \label{alg:pbe}
    \textbf{Input:} $M =\{x_1,\ldots, x_n\} \in \Rn{m}$ - training dataset.\\
    $Y=\{y_1,\ldots, y_n\} = \{\psi(x_1),\ldots, \psi(x_n)\} $ - training dataset after dimensionality reduction.\\
    $x$ - an out of sample data point.\\

    \textbf{Output:} $y$ - an out-of-sample extension of the data point $x$ that preserves the local properties of $\psi$.\\
    $err$ - an abnormality score of the data point $x$.
    \begin{algorithmic}[1]
        \STATE Find a set $N_\varepsilon(x)$ (Eq. \ref{neighborhood_epsilon}) of the nearest neighbors with radius $\varepsilon$ to the data point $x$ in $M$.
        \STATE For each data point $x_i \in N_\varepsilon(x)$, construct a weighted linear system
        $ W\psi(x_i) = y$ for $y$
        where the construction of $W$ is described in section \ref{sec:feature based variance}.
        \STATE $y$ is the optimal solution for the combined weighted linear system, as described in section \ref{sec:GLS}.
        \STATE When GLS is solved, find the residual $err$ of the solution.
    \end{algorithmic}
\end{algorithm}


\section{Geometric-based covariance matrix}\label{sec:feature based variance}
In this section we present a construction of $W$, which is  the
weight of the linear system for $y$, such that the resulted
out-of-sample extension $\hat{\psi}(x)$ agrees with principal
direction of its neighborhood in $\mathcal{N}$. The neighborhood
$N_\varepsilon(x)$ can be defined variously. In this paper, we use
the definition
\begin{equation}\label{neighborhood_epsilon}
N_\varepsilon(x)\triangleq\{y\in M : \norm{x-y}\leq\varepsilon\},
\end{equation}
for some positive $\varepsilon$, which ensures  locality of the
extension scheme. The parameter $\varepsilon$ should be fixed
according to the sampling density of $\mathcal{M}$ such that
$\abs{N_{\varepsilon}(x)}\geq d$. This restriction enables to detect the principal
directions of the image of $\psi(N_\varepsilon(x))$ in
$\mathcal{N}$.

In the rest of this section, we present the construction of $W$. The
first construction, presented in Section~\ref{sec:basic} provides  a
 mechanism to control the rate of influence of a data point $x_j\in
N_\varepsilon(x)$ on the value of $\hat{\psi}(x)$ as a function of
its distance from $x$. The second construction for $W$, presented in
Section~\ref{sec:tangent space}, incorporates information regarding
principal variance directions of $N_\varepsilon(x)$ such that the
resulted out-of-sample extension $\hat{\psi}(x)$
\lq\lq{}agrees\rq\rq{} with these directions.

\subsection{Distance based covariance matrix $W$}\label{sec:basic}
Although the definition of $N_{\varepsilon}(x)$ provides locality
for the scheme computation, it is reasonable to require that data
points in $N_{\varepsilon}(x)$, which are distant from $x$, affect
less than close data points. For this purpose, an \lq\lq{}affection
weight\rq\rq{}
\begin{equation}\label{eq:weights}
\lambda_j\triangleq\frac{1}{\norm{x-x_j}}
\end{equation}
 is assigned to each data point $x_j\in N_{\varepsilon}(x)$. Of course, any other decreasing function of the distance between $x$ and $x_j$ can be utilized.
 By defining the variance $\sigma_j$ to be proportional to the distance $\norm{x-x_j}$ such that $\sigma_j\triangleq\lambda_j^{-1}$, then we get a diagonal matrix $W$, whose $j$-th diagonal element is
\begin{equation}\label{eq:simple w}
w_j\triangleq\lambda_j^2I_d.
\end{equation}
Thus, due to Observation~\ref{obs:mahlanobis}, close data points in
$N_\varepsilon(x)$ affect $\hat{\psi}(x)$ more than data points that
are far away.

\subsection{Tangential space based covariance matrix $W$}\label{sec:tangent space}
In this section, we present a covariance matrix that encapsulates
geometric information concerning the manifold $\mathcal{N}$. The
covariance matrix $W$ is set such that the resulted extension obeys
the Lipschitz property of $\psi$.

Let $\mathcal{T}_j$ be the tangential space to $\mathcal{N}$ in
$\psi(x_j)$ and let $\mathcal{P}_j$ be the orthogonal projection on
$\mathcal{T}_j$. We denote the tangential component of $\omega_j$ by
$\omega_{j}^{t}=\mathcal{P}_j\omega_j$, and its orthogonal
complement by $\omega_{j}^{o}=(\mathcal{I}-\mathcal{P}_j)\omega_j$,
where $\mathcal{I}$ is the identity transformation.
Proposition~\ref{prop:tang-perp-comp} quantifies the tangential and
perpendicular components of $\omega_j$ from Eq.~\ref{eq:basic_rv},
as functions of the curvature of $\mathcal{N}$ in $x_j$ and
$\norminline{x-x_j}$.

\begin{proposition}\label{prop:tang-perp-comp}
Let $\norminline{x-x_j}\leq r$ and assume that the curvature of
$\mathcal{N}$ in $x_j$ is bounded by a constant $c_j$. If $\psi$ is
a Lipschitz function with constant $k$, then $\omega_j^t\leq kr$ and
$\omega_j^o\leq (c_jkr)^2$.
\end{proposition}

\begin{proof}
Without loss of generality it is assumed that $\psi(x_j) = 0 \in
\Rn{d}$ and $\mathcal{T}_j=\Rn{m}$. We denote the graph of the
manifold $\mathcal{N}$ in the neighborhood of $0$ by the function
$f:\mathcal{T}_j \rightarrow \Rn{d-m}$, where the data points in
$\mathcal{N}$ are $(z,f(z)),\,z\in\mathcal{T}_j$. Thus we get $f(0)
= 0$ and $\frac{\partial f}{\partial z}(0)=0$. Let $x\in\mathcal{M}$
be a data point in the neighborhood of $x_j$ and let
$\psi(x)=(z_x,f(z_x))$. Namely, $z_x=\mathcal{P}_j\psi(x)$ and
$f(z_x)=(\mathcal{I}-\mathcal{P}_j)\psi(x)$. Then, the Taylor
expansion of $f(z_x)$ around $0$ yields $f(z_x) = f(0) +
\frac{\partial f}{\partial z}(0) (z_x)+ O(\norminline{z_x}^2)$.
Since $\psi$ is assumed to be a Lipschitz function with constant
$k$, we get
$\norminline{z_x}=\norminline{\mathcal{P}_j(\psi(x)-\psi(x_j))}\leq
\norminline{\psi(x)-\psi(x_j)}\leq kr$. Thus, we get that
$\norminline{\omega_j^t}=\norminline{\mathcal{P}_j(\psi(x)-\psi(x_j))}\leq
kr$ and $\norminline{\omega_j^o}\leq (c_jkr)^2$.
\end{proof}

From Eq.~\ref{eq:basic_rv}, Proposition~\ref{prop:tang-perp-comp}
provides a relation between the tangential and perpendicular
components of $\omega_j$ Thus, $\Omega$ from Eq.~\ref{eq:system_rv}
is the $kd$-long vector, whose $j$-th section is the $d$-long vector
$(\omega^t_j,\ldots,\omega^t_j,\omega^o_j,\ldots,\omega^o_j)^T$,
where its first $m$ entries are the tangential weights and the rest
$d-m$ are the perpendicular weights. The corresponding covariance
matrix is the $kd\times kd$ blocks diagonal matrix $W$, whose $j$-th
diagonal element is the $d\times d$ diagonal matrix
$$
w_j = \begin{pmatrix}
\lambda_j^{2}& & & & &\\
 &\ddots& & & &\\
& &\lambda_j^{2}& & & \\
& & &(c\lambda_j)^{4}& &\\
 &&&&\ddots &\\
&&&&&(c\lambda_j)^{4} \end{pmatrix},
$$
where the first $m$ diagonal elements are $\lambda_j^2$ (see
Eq.~\ref{eq:weights}), and the rest $d-m$ are
$(c\lambda_j)^{4}$.Then, the solution is given by
Eq.~\ref{eq:system_rv} while $J$ and $\Psi$ remain the same.

\subsubsection{Tangential Space Approximation}
In real life applications, in order to use the tangential space to
$\mathcal{N}$ in $\psi(x_j)$, it has to be first approximated by
using its neighboring data. In this section,  we approximate the
tangential space and the principle directions of the manifold
$\mathcal{N}$ at $\psi(x_j)$. Then, these approximations are
incorporated in the construction of $W$ to ascribe heavy weights to
the tangential direction and less significant weights to the
perpendicular ones. The principle directions of the data and the
variance of each direction are the eigenvectors and eigenvalues of
the covariance matrix  of a data points $\psi(x_j)$, respectively.
This covariance matrix is also known as the PCA matrix. A
multi-scale version of the local PCA algorithm is analyzed
in~\cite{MA} and it can be used in our analysis. It is important to
take at least as many data points as the dimensionality of
$\mathcal{N}$.

The covariance matrix of a data point $\psi(x_j)$ is computed in the
following way: for simplicity of calculations, we take as the set of
neighbors of $\psi(x_j)$ the set  $\psi(N_{\varepsilon_{1}}(x)) =
{\psi(x_1), \ldots ,\psi(x_k)}$. Then, we form the $k\times d$
matrix $X$, whose rows in the aforementioned set are $$
X=\begin{pmatrix} - \psi(x_1)  -\\ \vdots
\\- \psi(x_k) -\end{pmatrix}.
$$
Accordingly, we define
\begin{equation}\label{eq:cov xj}
cov\left(\psi(x_j)\right) \triangleq
\left(\frac{1}{\varepsilon_1^2}\right)\left(\frac{1}{k}\right)X
X^{t}, ~~\mbox{ for all } i=1, \ldots k.
\end{equation}
Since we take the same set of data points then for all $i$ and $j$
we have $cov(\psi(x_j)) = cov(\psi(x_i)) $. To make the calculation
and stability issues easier we add the $(c \cdot \lambda_i)^{4}$
component to all the diagonal components. Consequently, we define:
\begin{equation}\label{eq:tang weights}
w_j \triangleq\left(\lambda_j^{-2} cov(\psi(x_j)) + \begin{pmatrix} (c \cdot
\lambda_j)^{-4} & 0 & \ldots&0\\ 0  & (c \cdot \lambda_j)^{-4}  &
\ldots&0\\ & \vdots& \ddots &\\0&\ldots&0&(c \cdot \lambda_j)^{-4}
\end{pmatrix}\right)^{-1}.
\end{equation}

Since $w_i$ is positive definite, it is invertible. We notice that
it was possible to add the $(c \cdot \lambda_i)^{-4}$ weight
components only to the least significant directions of the
covariance matrix by computing the SVD \cite{MA} of the covariance
matrix. This does not improve the accuracy significantly but adds
more complexity to the computation.  $W$ is a block diagonal matrix
with the same structure as appears in Eq.~\ref{eq:W}.

Another option is to make different estimations for the tangential
space in different data points by using different sets of data
points in the covariance matrix computation. While this estimation
should be more accurate, it is more computationally expensive.

\section {Bounding the error of the out-of-sample extension}
\label{sec:bound}

In this section, we prove that the error of the out-of-sample
extension is bounded in both cases of distance-based weights
(Eq.~\ref{eq:simple w}) and the tangential-based
weights(Eq.~\ref{eq:tang weights}). It means that for any function
$\psi : \mathcal{M}  \rightarrow \mathcal{N}$, which agrees on a
given set of data points and satisfies certain conditions, the
out-of-sample extension $\hat{\psi}(x)$ of the data point $x$  is
close to $\psi (x)$.

First, we prove the consistency of the algorithm. In other words,
the out-of-sample extension of data points, which  coverage to an
already known data point, will converge to its already known image.

\begin{lem}
\label{lemma:5}
 Assume $x\in \mathcal{M}$. If $x \rightarrow x_j
\in M$ then $\hat\psi(x)\rightarrow \psi(x_j)$.
\end{lem}

An intuition for the proof of Lemma \ref{lemma:5} is that the
distance from the point $x_i \in M$ is  inversely proportional to
the weight of the equation $y = \psi(x_i)$ in
 Eq.~\ref{eq:system_rv}, therefore, when $x \rightarrow x_i $, the
distance tends to $0$ and the weight tends to $\infty$. Notice that
when $x=x_i$ then, according to Eq.~\ref{eq:weights}, $\lambda =
\infty$ and the out-of-sample extension is undefined.

\begin{defa}
The dataset $M\subset\mathcal{M}$ is called a $\delta$-net of the manifold $\mathcal{M}$ if for any data point $x \in \mathcal{M}$ there is $ \tilde x \in M$ such that $\norminline{x-\tilde x} \leq \delta$.
\end{defa}

\begin{theorem}
Assume that $M$ is a $\delta$-net of $\mathcal{M}$.  Let $\psi:\mathcal{M} \rightarrow \mathcal{N}$ be a Lipschitz function  with a constant $K$. If $\varepsilon_1 = \delta$ and $\hat\psi(x)$ is computed using the weights in Eq.~\ref{eq:simple w},  then $\norminline{\hat\psi(x)- \psi(x)} \leq 3K\delta$.
\end{theorem}

\begin{proof}
We denote by $ N_{\delta}(x) = \{x_1,...x_k\}$ the set of data points in the $\varepsilon_1
=\delta $ neighborhood of $x$.  It is easy to see that all the data
points of  $\psi(x_i)$ are inside a ball $B \subset \mathcal{N}$ of
radius $K\varepsilon_1$. Therefore, the out-of-sample extension $y$ is also in
this ball, namely
\begin{equation} \label{eq:th1es1}
\norm{\hat\psi(x)-\psi(x_i)}<2K\varepsilon_1 .
\end{equation}
Since $\psi$ is a Lipschitz function  and $\norminline{x-x_i}<\varepsilon_1$,
we have
\begin{equation} \label{eq:th1es2}
\norm{\psi(x_i)-\psi(x)}<K\varepsilon_1 .
\end{equation}
By combining Eqs. \eqref{eq:th1es1} and \eqref{eq:th1es2}, we get $
\norm{\hat\psi(x)-\psi(x)}  \leq
\norm{\hat\psi(x)-\psi(x_i)}+\norm{\psi(x_i)-\psi(x)}  \leq
3K\varepsilon_1$.

 \end{proof}

Next, we show an identical result for the case where the weights
from Eq.~\ref{eq:tang weights} are utilized to construct the
covariance matrix $W$. Moreover, the approximations of the
tangential spaces converge to the correct tangential space as
$\varepsilon_1$ tends to $0$.

\begin{thm} \label{thm:main}
Let $M$ be a $\delta$-net of $\mathcal{M}$ and let $\psi:\mathcal{M} \rightarrow \mathcal{N}$  be a Lipschitz function with a constant $K$. If $\varepsilon_1 = \delta$ and $\hat\psi(x)$ is computed using the weights in Eq.~\ref{eq:tang weights},  then $\norminline{\hat\psi(x)- \psi(x)} \leq 3K \delta$.
\end{thm}

\begin{proof}
Let $ N_{\delta}(x) = \{x_1,...x_k\}$ be the  $\delta $ neighborhood
of $x$. Then, the weight matrix becomes
$$
W = \begin{pmatrix} w_1& 0 & \ldots&0\\ 0  & w_2& \ldots&0\\ &
\vdots& \ddots &\\0&\ldots&0&w_k \end{pmatrix} .
$$

By using Eq.  \eqref{eq:psi_hat_solution}, we get
\begin{equation} \label{eq:yinw_i}
\hat\psi(x) = \left(\sum{ w_i} \right)^{-1}  \sum{w_i \psi(x_i)},
\end{equation}
where the $w_i$ matrices are defined in Eq.~\ref{eq:tang weights}.
The structure of the $w_i$ matrices allows us to find a basis in
which all the matrices $w_{i}$ become diagonal. Let us denote the diagonal form of $w_i$ by $D_i$. Then $D_i = T w_{i} T^{-1}$ where $T$ is the transformation matrix. We can
rewrite Eq. \eqref{eq:yinw_i} to become
$$
\begin{array}{lll}
T\hat\psi(y) & = & T\left(\sum{ w_i} \right)^{-1}T^{-1}T  \sum{w_i \psi(x_i)} \\
& = & \left(\sum{T w_i T^{-1}} \right)^{-1}  \sum{T w_i T^{-1}T \psi(x_i)} \\
& = & \left(\sum{D_i} \right)^{-1}  \sum{D_i T \psi(x_i)} .
\end{array}
$$
Since all $D_i$ are diagonal,  we get a weighted average of the data
points $\psi(x_i)$ in the new basis, which is known to be in convex hull.
Thus, it is located inside a ball that contains all the data points. It
means that $\hat\psi(x)$ is inside a ball of radius $K\varepsilon_1$ that
contains all $\psi(x_i)$. Therefore,
$$
\begin{array}{lll}
\|\hat\psi(x)-\psi(x)\| & = & \|\hat\psi(x)-\psi(x_i)+\psi(x_i)-\psi(x)\| \\
& \le & \|\hat\psi(x)-\psi(x_i)\|+\|\psi(x_i)-\psi(x)\| \\
& \le &  2K\varepsilon_1+ K\varepsilon_1= 3K\varepsilon_1 =3K\delta .
\end{array}
$$

\end{proof}

\section{Out-of-sample extension complexity}
\label{sec:complexity} Recall that the dataset $M$ consists of $p$
data points and assume that the number of data points in the
neighborhood of $x$ is $k$. The covariance matrix of a data point
$\psi(x_j)$ from Eq.~\ref{eq:cov xj} is also computed once  for each
data point in $M$, considering each of its $k$ neighbors. The
complexity of the neighborhood computation is $O(p)$ operations.
Then, the covariance matrix is computed in $O(dk^2)$ operations.
Thus, the complexity of this pre-computation stage is $O(p \cdot
(p+dk^2))=O(p^2)$ operations. For each data point, we multiply
vectors of size $d\cdot k$ by  matrices of size $k\times d\cdot k$
or $k \times k$. Thus, the out-of-sample extension complexity is
$O(k^2 \times d^2)$ operations.


\section{Experimental results}
\label{sec:examples}
\subsection{Example I: Data points on a sphere}
The function $\psi: [0,\pi] \times [0,\pi] \rightarrow
\mathbbm{R}^3$ maps the spherical coordinates $(\phi, \theta)$ into
a 3-D sphere of radius $1$. More specifically, $ \psi(\phi, \theta)
= \left(\sin(\phi)\cos(\theta), \sin(\phi)\sin(\theta),
\cos(\phi)\right). $ We generate $900$ data points angularly equally
distributed where we have $30$ data points  on each axis as a
training dataset.
 We generate 100 random data points for which we compute the out-of-sample extension.
 The results from the application of the  algorithm using weights as defined in Eq. \ref{eq:simple w}, are shown in Fig. \ref{fig:sphere}.
In Fig. \ref{fig:sphere_3estimations}, we can see three different
results from an out-of-sample extension using different weights as
presented in Section \ref{sec:feature based variance}. In Table
\ref{algorithm_error}, we show how the results get better for more
advanced weight algorithms. We display an accurate error mean for
the algorithm. We also show the improvement  of the results when we
take $2500$ data points angularly equally distributed with $50$ data
points on each axis:

\begin{table}[H]
\centering
    \begin{tabular}{ | p{5cm} | p{2cm} | p{3cm} | p{3cm} |}
    \hline
    Algorithm type  & Color in Fig. \ref{fig:sphere_3estimations}& Mean error for 900 data points & Mean error for 2500 data points \\ \hline
    Weights as in Eq. \ref{eq:simple w} & Yellow & $1.04 \cdot10^{-2}$ & $6.01 \cdot10^{-3}$ \\ \hline

    Weights as in Eq. \ref{eq:tang weights} & Red & $8.08  \cdot10^{-3}$ & $4.45  \cdot10^{-3}$ \\ \hline

    Weights as in Eq. \ref{eq:tang weights} but with different estimations for the tangential space at each data point & Black & $6.14  \cdot10^{-3}$  & $3.17  \cdot10^{-3}$ \\
    \hline
    \end{tabular}
\caption{The mean error performances of the algorithms for different
number of data points and different weights} \label{algorithm_error}
\end{table}

\begin{figure}[H]
\includegraphics[width=15cm,keepaspectratio]{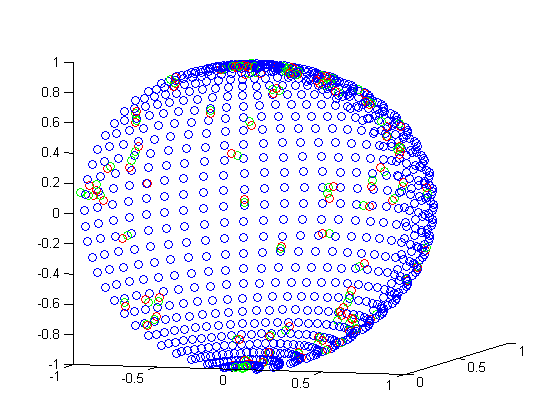}
\caption{An illustration of the out-of-sample extension algorithm on
a sphere. Blue - the original data set, green - the correct images,
red - the out-of-sample extension calculated using the algorithm
with weights from Eq. \ref{eq:simple w}.} \label{fig:sphere}
\end{figure}

\begin{figure}[H]
\includegraphics[width=15cm,keepaspectratio]{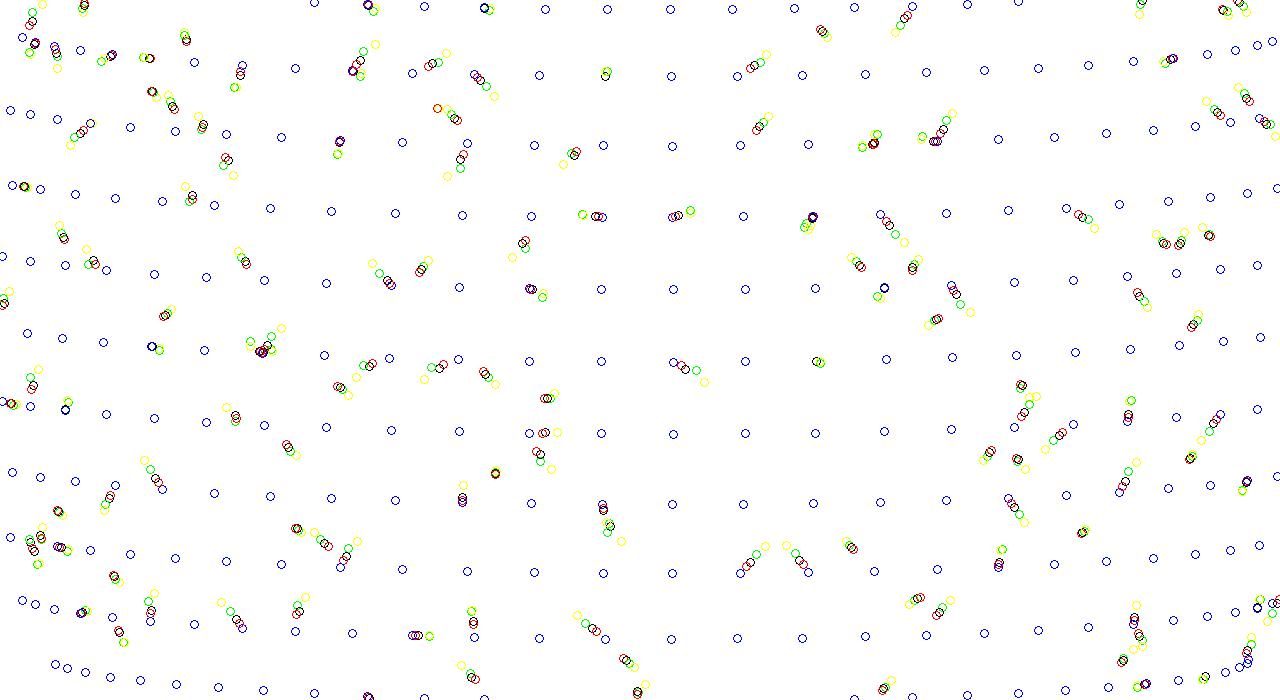}
\caption{An illustration of the algorithms on a sphere in Table
\ref{algorithm_error}. Blue - the original data set, green - the
correct images. yellow - the out-of-sample extension computed  using
the algorithm with weights from Eq. \ref{eq:simple w}. Red - the
out-of-sample extension computed  using the algorithm with weights
from Eq. \ref{eq:tang weights}. Black - the out-of-sample extension
computed using the weights from Eq. \ref{eq:tang weights}, but with
different estimations for the tangential space at each data point.}
\label{fig:sphere_3estimations}
\end{figure}

\subsection{Example II: Dimensionality reduction example}

DARPA datasets \cite{DP99} from 1998 and 1999 are utilized here  to
find anomalies in them. All the activities and non-activities are
labeled and published. These datasets contain different types of
cyber attacks that we consider as anomalies.

We use this dataset to evaluate the performance of the out-of-sample
extension using weights from Eq. \ref{eq:tang weights} and the
Mahalanobis distance from Eq.~\ref{eq:maha measure}. The experiment
done by following the example in \cite{gild:PHD}. We use the same
data and same mapping that was developed in \cite{gild:PHD}.
Diffusion Maps (DM) \cite{coifman:DM}, which was applied to DARPA
data, reduces the dimensionality by embedding $\mathbbm{R}^{14}$ to
$\mathbbm{R}^{5}$ such that $\psi:\mathbbm{R}^{14} \rightarrow
\mathbbm{R}^5$.

We present two experiments using this data to evaluate the
performance of the out-of-sample extension. The first experiment is
an out-of-sample extension for non-anomalous data points by
comparing the original results from the DM embedding. The second
experiment evaluates the anomaly detection of the algorithm.

\subsubsection{Out-of-sample extension on DARPA data}
In this experiment, we use 800 data points and an embedding function
$\psi:\mathbbm{R}^{14} \rightarrow \mathbbm{R}^5$ that was described
before. By taking a random subset $\{x_1, \ldots, x_k\}$ of data
points and by using the values $\{\psi(x_1), \ldots, \psi(x_k)\}$,
we approximate $\psi$ on $50$ data from the 800. We compare the
approximated result to the correct values of $\psi$ and measure the
error. To evaluate the performance of our method,  we compare these
results to the results from other leading methods such as the
classic Nystr\"om method, the Multiscale data sampling and function extension (MSE) method described in
\cite{sampta,bermanis:multiSampExtACHA} and the auto-adaptative
Laplacian Pyramids method described in
\cite{rabin2012heterogeneous}.

To make the presentation self contained,  Nystr\"om, MSE and
auto-adaptative Laplacian pyramids methods are outlined next.

\begin{description}
\item[Nystr\"om method:]
The Nystr\"om method \cite{baker,num_rec} is vastly used for an
out-of-sample extension in dimensionality reduction methods. It is a
numerical scheme for the extension of integral operator
eigenfunctions. It finds a numerical approximation for the
eigenfunction problem
\begin{equation}
 \int_a^b G(x,y) \phi (y) dy = \lambda \phi (x)
 \label{NYS1}
\end{equation}
where $\phi$ is an eigenfunction and $\lambda$ is the corresponding
eigenvalue. Given a set of equidistant points $\{x_{j}\}_{j=1}^{n}
\subset[a,b]$. Assume that $G$ is similarity matrix that is defined on the data whose $(i, j)$ entry measures the similarity between the data points $x_{i}$ and $x_{j}$, namely
\begin{equation}\label{eq:G}
G \triangleq \left[\begin{array}{cccc}
g\left(x_1,x_1\right) & g\left(x_1,x_2\right) & \cdots & g\left(x_1,x_n\right)\\
g\left(x_2,x_1\right) & g\left(x_2,x_2\right) & \cdots & g\left(x_2,x_n\right)\\
\vdots & \vdots & \ddots & \vdots\\
g\left(x_n,x_1\right) & g\left(x_n,x_2\right) & \cdots & g\left(x_n,x_n\right)\end{array} \right].
\end{equation}
A Gaussian function is a popular choice for $g$, and it is given by
\begin{equation}\label{eq:gaus}
g_\epsilon\left(x,x'\right)\triangleq \exp\left(-\left\Vert x-x'\right\Vert^2/\epsilon\right),
\end{equation}
where $\left\Vert\cdot\right\Vert$ constitutes a metric on the space.
Then, Eq. \ref{NYS1} can be approximated by a quadrature
rule to become $\frac{b-a}{n}\sum_{j=1}^{n} G(x_{i},x_{j})\phi
(x_{j})=\lambda \phi (x_{i})$. Then, the Nystr\"om extension of
$\phi$ to a new data point $x_\ast$ is $
\hat{\phi}(x_\ast)\stackrel{\Delta}{=}
\frac{b-a}{n\lambda}\sum_{j=1}^{n} G(x_\ast,x_{j})\phi (x_{j})$.

If $G$ is symmetric, then its normalized eigenfunctions
$\left\{\phi_i\right\}_{i=1}^n$ constitute an orthonormal basis to
$\mathbb{R}^n$. Thus, any vector $f=\left[f_1\,f_2\, \ldots \,
f_n\right]^T$, ($f_j=f\left(x_j\right),\,j=1,\ldots,n$)  can be
decomposed into a superposition of its eigenvectors
 $f=\sum_{i=1}^n \left(f^T\cdot \phi_i\right)\phi_i$.
Then, the Nystr\"om extension of $f$ to $x_\ast$ becomes $f_\ast
\triangleq \sum_{i=1}^n \left(f^T\cdot
\phi_i\right)\hat{\phi_i}\left(x_\ast\right)$.

\par\noindent

\item[MSE method:]
\mbox{}


\begin{algorithm}[H]
\caption{Randomized interpolative decomposition\protect \\
\textbf{Input:} An $m\times n$ matrix $A$ and an integer $l$, s.t.
$l<\min\{m,n\}$.\newline \textbf{Output:} An $m\times l$ matrix $B$
and an $l \times n$ matrix $P$ that satisfies $\left\Vert
A-BP\right\Vert\lesssim l\sqrt{mn}\sigma_{l+1}(A)$}.
\label{alg:rand_ID}
\begin{algorithmic}[1]

\STATE Use a random number generator to form a real $l \times n$
matrix $G$ whose entries are i.i.d Gaussian random variables of zero
mean and unit variance. Compute the $l \times n$ product matrix $
W=GA$.

\STATE Apply the pivoted QR routine to $W$ (Algorithm 5.4.1 in
\cite{golub}), $WP_R=QR$, where $P_R$ is an $n\times n$ permutation
matrix, $Q$ is an $l \times l $ orthogonal matrix, and $R$ is an $l
\times n$ upper triangular matrix, where the absolute values of the
diagonal are ordered decreasingly.

\STATE Split $R$ s.t.
\begin{equation*}
R=\left(\begin{array}{c|c}R_{11}&R_{12}\\
\hline 0&R_{22}\end{array}\right),
\end{equation*}

where $R_{11}$  is $l \times l $, $R_{12}$ is $l \times \left(n-l
\right)$ and $R_{22}$ is $\left(k  - l \right)\times\left(n-l
\right)$.

\STATE\label{step4} Define the $l \times l $ matrix $S=QR_{11}$.

\STATE\label{sampling} From Step \ref{step4}, the columns of $S$ constitute a subset of the columns of $W$. In other words, there exists a finite sequence $i_1,i_2,...,i_{l-1},i_{l}$ of integers such that, for any $j=1,2,...,l-1,l$, the $j$th column of $S$ is the $i_j$th column of $W$. The corresponding columns of $A$ are collected into a real $n\times l$ matrix $B$, so that, for any  $j=1,2,...,l-1,l$, the $j$th column of $B$ is the $i_j$th column of $A$. Then, the sampled dataset is $D_s=\left\{x_{i_1},x_{i_2},...,x_{i_{l-1}},x_{i_{l}}\right\}$.\\

\end{algorithmic}
\end{algorithm}

\pagebreak

\begin{algorithm}[H]
\caption{Single-scale extension\protect \\
\textbf{Input:} $N\times l_s$ matrix $B_s$, the associated sampled
data $D_s=\left\{x_{i_1},x_{i_2},...,x_{i_{l-1}},x_{i_{l}}\right\}$,
a new data point $x$ and  a function $\bar{f}=\left(f(x_1)\,
f(x_2)\,\ldots\,f(x_n)\right)^T$ to be extended.  \newline
\textbf{Output:} The projection $\bar{f}_s =
\left(f_s(x_1),f_s(x_2),\ldots,f_s(x_n)\right)^T$ of $f$  on the
numerical range of the associated kernel matrix, its extension
$f_s(x)$ to $x$,  and the sampled dataset $D_s$.} \label{alg: sse}
\begin{algorithmic}[1]

%

\STATE Apply SVD to $B_s$, s.t. $B_s = U_s\Sigma_s V_s^{\ast }$.

\STATE Calculate the pseudo-inverse
$B_s^{\dagger}=V\Sigma^{-1}U^{\ast}$ of $B_s$.

\STATE Calculate the coordinates vector
$c=(c_1,c_2,\ldots,c_{l_s})^T=B_s^{\dagger}f$ of the orthogonal
projection of $\bar{f}_s$ on the range of $B_s$ in the basis of
$B_s$'s columns.

\STATE \label{sse:step6} Calculate the orthogonal projection
$\bar{f_s}= B_s c$ of $f$ on $B_s$. .

\STATE \label{sse:step7} Calculate the extension of $\bar{f}_s$ to
$x$ s.t. $f_s\left(x\right)=\left(g_s(\Vert x-x_{s_1}\Vert),
g_s(\Vert x-x_{s_2}\Vert), \ldots, g_s(\Vert
x-x_{s_{l_s}}\Vert)\right) c$.
\end{algorithmic}
\end{algorithm}

\begin{algorithm}[H]
\caption{MSE\protect \\
\textbf{Input:} A dataset $D=\{x_1,\ldots,x_n\}$ in $\mathbb{R}^d$,
a positive number $T>0$, a new data point
$x\in\mathbb{R}^d\backslash D$, a function $\bar{f}=\left(f(x_1)\,
f(x_2)\,\ldots\,f(x_n)\right)^T$ to be extended  and an error
parameter $err\geq 0$.
\newline \textbf{Output:} An approximation $\bar{G} =
\left(G(x_1),G(x_2),\ldots,G(x_n)\right)^T$ of $f$ on $D$  and its
extension $G(x)$ to $x$.}\label{alg:ME}
\begin{algorithmic}[1]

\STATE Set the scale parameter $s=0$, $\bar{F}_{-1}=0$ and
$F_{-1}(x)=0$.

\WHILE{$\left\Vert \bar{f}-\bar{F}_{s-1}\right\Vert > err$}

\STATE Form the Gaussian kernel $K_s$ on $D$ (see
$\left(K_{\epsilon}\right)_{ij}=g_\epsilon\left(\left\Vert
x_i-x_j\right\Vert\right),\,i,j=1,\ldots,N$), with
$\epsilon_s=\frac{T}{2^s}$.

\STATE Estimate the numerical rank $l_s$ of $K_s$ using
$R_\delta\left(K_\epsilon\right)\leq \prod_{i=1}^d C(\vert
I_i\vert,\epsilon,\delta)$.

\STATE Apply Algorithm \ref{alg:rand_ID} to $K_s$ and $l_s$ to get
an $n\times l_s$ matrix $B_s$ and sampled dataset $D_s$.

\STATE Apply Algorithm \ref{alg: sse} to $B_s$ and $\bar{f}$. We get
the approximation $\bar{f}_s$ to $\bar{f}-\bar{F}_{s-1}$ at scale
$s$, and its extension $f_s(x)$ to $x$.

\STATE Set $\bar{F_s}=\bar{F}_{s-1}+\bar{f}_s, \,
F_s(x)=F_{s-1}(x)+f_s(x)$, $s=s+1$.

\ENDWHILE \STATE $\bar{G}=\bar{F}_{s-1}$ and $G(x)=F_{s-1}(x)$.
\end{algorithmic}
\end{algorithm}

\item[Laplacian Pyramid method:]
 The Laplacian pyramid is a multi-scale
algorithm for
 extending an empirical function $f$, which is
defined on a dataset $\Gamma$, to new data points. Mutual distances
between the data points in $\Gamma$ are used to approximate  $f$ in
different resolutions.

  $\Gamma$ is a set of $n$ data
points in $\mathbb{R}^{m}$ and  $f$ is  a function  defined on
$\Gamma$. A Gaussian kernel is defined on $\Gamma$ as $W_0
\triangleq w_0(x_i,x_j) =
e^{\frac{-\|x_{i}-x_{j}\|^{2}}{\sigma_0}}$.
 Normalizing
$W_0$ by $K_0 = k_0(x_i,x_j) = q_0^{-1}(x_i)w_0(x_i,x_j)$ where $
q_0(x_i) = \sum_j{w_0(x_i,x_j),}$
 yields a smoothing operator $K_0$. At a finer scale
$l$, the Gaussian kernel $ W_l = w_{l}(x_i,x_j) =
e^{{-\|(x_{i}-x_{j})\|^{2}}/({\frac{\sigma_0}{2^l}})}$ yields the
smoothing operator $ K_{l}= k_{l}(x_i,x_j) =
q_{l}^{-1}(x_i)w_{l}(x_i,x_j)$.

For any function $f: \Gamma \rightarrow \mathbb{R}$, the Laplacian
Pyramid representation of $f$ is defined iteratively as follows:
\begin{equation}\label{LPyramid_APP}\begin{split}
s_0(x_k) = \sum_{i=1}^{n}k_{0}(x_i,x_k)f(x_i)
\;\;\;\;\;\;\;\;\;\;\;\;\; \mbox{ for level} \;\;l=0
\\
s_l(x_k) = \sum_{i=1}^{n}k_{l}(x_i,x_k)d_{l}(x_i)
\;\;\;\;\;\;\;\;\;\;\;\;\;\;\;\;\;\;\;\; \mbox{ otherwise}.\\
\end {split}
\end{equation} The differences
\begin{equation}
\begin{split}
d_1 = f - s_0 \;\;\;\;\;\;\;\;\;\;\;\;\; \mbox{ for level} \;\;l=1
\\
d_l = f - \sum_{i=0}^{l-1} s_i \;\;\;\;\;\;\;\;\;\;\;\;\; \mbox{ for
level} \;\;l\ge 1
\end{split}
\end{equation}
are input for this algorithm at level $l$.

 Equation \eqref{LPyramid_APP} approximates a given function $f$
in a multi-scale manner, where $f \approx s_0 + s_1 + s_2 + \cdots$.
An admissible error should be set a-priori and the iterations in Eq.
\eqref{LPyramid_APP} stop when $\|f - \sum_{k} s_k \| < \mbox{err}.$

We extend $f$ to a new point $y \in \mathbb{R}^{m} \backslash
\Gamma$ in the following way:

\begin{equation}\label{LPyramid_EX}\begin{split}
s_0(y) = \sum_{i=1}^{n}k_{0}(x_i,y)f(x_i) \;\;\;\;\;\;\;\;\;\;\;\;\;
\mbox{ for level} \;\;l=0
\\ s_l(y) = \sum_{i=1}^{n}k_{l}(x_i,y)d_{l}(x_i)
\;\;\;\;\;\;\;\;\;\;\;\;\;\;\;\;\;\;\;  \mbox{ otherwise}.
\end {split}
\end{equation}

The extension of $f$ to the point $y$ is evaluated from  Eq.
\eqref{LPyramid_EX}  as $f(y) = \sum_{k}s_k(y)$.

\end{description}

 The performance results of the 4 methods are shown in
Fig.\ref{fig:oos_comperison}.

\begin{figure}[H]
    \includegraphics[width=15cm,keepaspectratio]{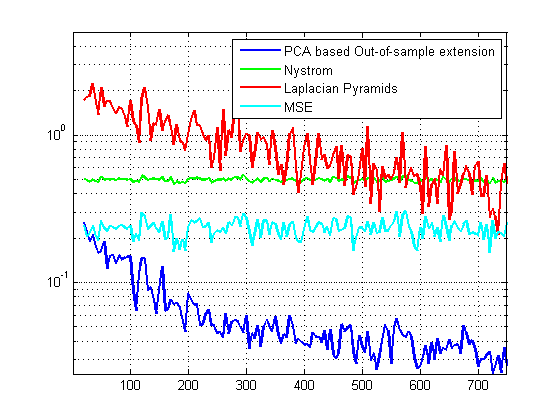}
    \caption{Performance comparison between different out-of-sample extension methods. The x-axis is the size of the training set and
     the y-axis is the norm of the error from  the out-of-sample extension  methods.}
    \label{fig:oos_comperison}
\end{figure}
\subsubsection{Anomaly detection on DARPA data}

In this experiment, all the 1321 available  data points are used as
our training dataset. We show in Fig. \ref{fig:darpa_orig_image} the
image of these data points after the embedding by $\psi$. The normal
behaved manifold in the embedded space in Fig.
\ref{fig:darpa_orig_image} has the ``horseshoe" shape. We can see a
few data points, which are classified as anomalous, are the labeled
attacks. Then, a set of newly arrived data points are assigned with
coordinates in the embedded space via the application of Nystr\"om
extension as can be seen in the left image in Fig.
\ref{fig:darpa_new_image}. It is also done by the applicaion of the
MSE  algorithm in \cite{bermanis:multiSampExtACHA}. Data point \#51,
which is a newly arrived data point, is an anomalous that can be
seen as an outlier on the left side of the normal (``horseshoe")
manifold.

We apply our out-of-sample extension algorithm using weights from
Eq. \ref{eq:tang weights},  to the same set of newly arrived data
points.  The results are shown on the right image in Fig.
\ref{fig:darpa_new_image}. To find anomalies, we compute the
Mahalanobis distance of the extension using Eq. \eqref{eq:maha
measure}  for each of the newly arrived data points. We see that
data point \#$51$ emerged as having a much higher residual error
($2.72 10^{-7}$) than the other data points whose average residual
error is $6.21 10^{-10}$. All the Mahalanobis distance values are
shown in Fig. \ref{fig:residuals}.

\begin{figure}[H]
\includegraphics[width=15cm,keepaspectratio]{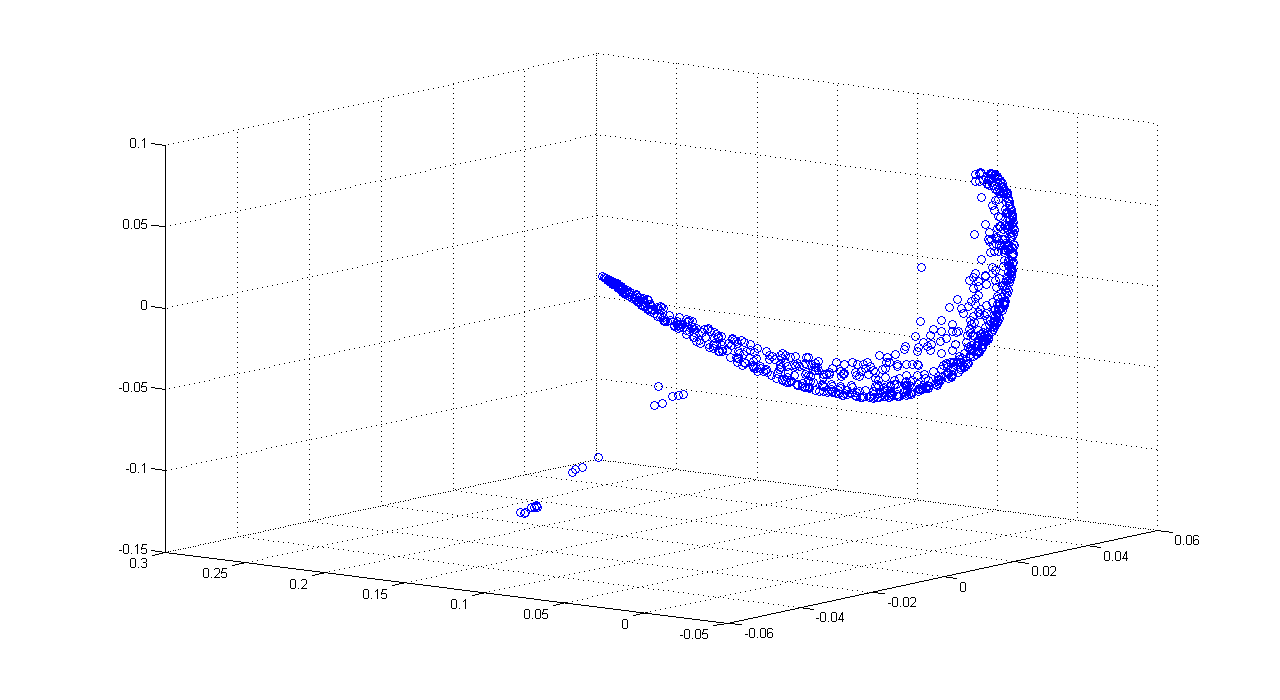}
\caption{The first three coordinates of the data points in
$\mathbb{R}^14$  after its embedding into $\mathbb{R}^5$. }
\label{fig:darpa_orig_image}
\end{figure}

\begin{figure}[H]
\includegraphics[width=15cm,keepaspectratio]{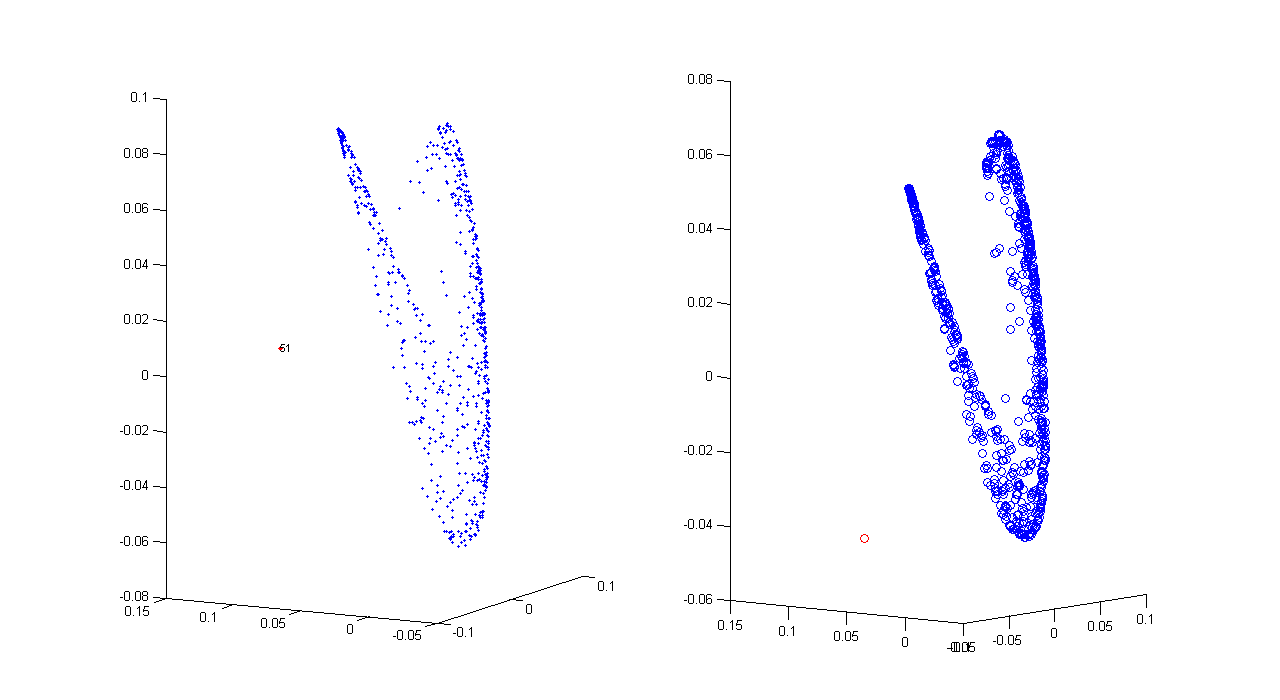}
\caption{Out-of-sample extension computed for a new day. In the left
side, the out-of-sample extension is computed via the application of
the Nystr\"om extension algorithm. Data point \#51 is known to be an
anomalous data point. In the right image, the output of the
algorithm, which uses  weights from Eq. \ref{eq:tang weights}, is
presented by the red data point which is the data point \# 51.}
\label{fig:darpa_new_image}
\end{figure}

\begin{figure}[H]
\includegraphics[width=15cm,keepaspectratio]{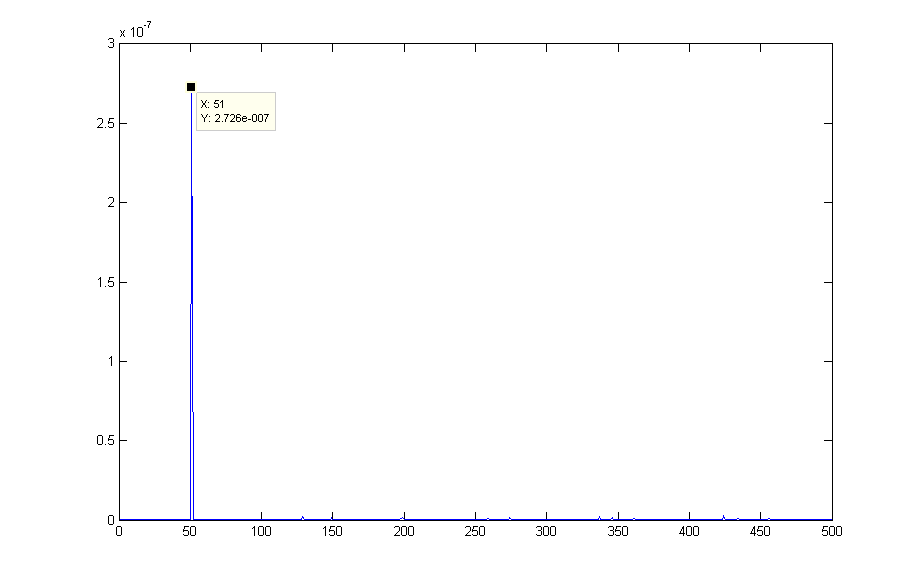}
\caption{The  Mahalanobis distance values. We see that data point
\#51 has the highest value. Therefore, it is classified as an
anomalous data point.} \label{fig:residuals}
\end{figure}


\section*{Conclusions}
In this paper, we present an efficient out-of-sample extension
(interpolation) scheme for dimensionality reduction maps that are
widely used in the field of data analysis. The computational cost of
such maps is high.  Therefore, once such a map is computed over a
training set, an efficient extension scheme is needed. The presented
scheme is based on the manifold assumption, which is widely used in
the field of dimensionality reduction. It provides an optimal
solution from a stochastic geometric-based linear equations system
that is determined by the application of  local PCA  of the embedded
data. Moreover, the scheme enables to detect abnormal data points.
The interpolation error was analyzed by assuming that the original
map is a Lipschitz function. The scheme was applied to both
synthetic and real-life data to provide good results by mapping data
from the manifold to the image manifold and by detection of abnormal
data points.

\section*{Acknowledgments} This research was partially supported by the
Israeli Ministry of Science \& Technology (Grants No. 3-9096,
3-10898), US-Israel Binational Science Foundation (BSF 2012282),
Blavatnik Computer Science Research Fund and ICRC Funds.

\bibliographystyle{plain}

\end{document}